\def\year{2015}
\documentclass[letterpaper]{article}
\usepackage{aaai}
\usepackage{times}
\usepackage{helvet}
\usepackage{courier}
\usepackage{graphicx,amsmath,amssymb,amsfonts,bm,epsf,url,dsfont}

\usepackage{amsthm}
\usepackage{color}
\usepackage{graphicx,amsmath,amssymb,amsfonts,bm,epsf,url,dsfont}
\usepackage{bbm}
\usepackage{booktabs}
\usepackage{cases}
\usepackage[small,bf]{caption}
\usepackage{algorithm}
\usepackage{algorithmic}
\usepackage{subcaption}
\usepackage{amsmath}
\usepackage{amssymb}
\usepackage{mathtools}
\usepackage{paralist}
\usepackage{sidecap}

\newcommand{\citep}[1]{\cite{#1}}
\newcommand{\namecite}[1]{\citeauthor{#1}~(\citeyear{#1})}

\newcommand{\lemref}[1]{Lemma~\ref{#1}}
\newcommand{\thmref}[1]{Theorem~\ref{#1}}

\newcommand{\propref}[1]{Proposition~\ref{#1}}

\newcommand{\appref}[1]{Appendix~\ref{#1}}

\newcommand{\eqnref}[1]{Equation~\ref{#1}}

\newcommand{\tblref}[1]{Table~\ref{#1}}
\newcommand{\algref}[1]{Algorithm~\ref{#1}}

\newcommand{\Sset}{\mathcal{S}}

\newcommand{\Aset}{\mathcal{A}}
\newcommand{\Eset}{\mathcal{E}}

\newcommand{\Mset}{\mathcal{M}}

\newcommand{\Nset}{\mathbb{N}}

\newcommand{\Aalg}{\mathbf{A}}
\newcommand{\Event}{\mathcal{E}}
\newcommand{\actProbe}{\mathsf{P}}
\newcommand{\actSkip}{\mathsf{S}}

\newcommand{\algfont}[1]{\textsc{#1}}

\newcommand{\cf}{\textit{c.f.}}
\newcommand{\eg}{\textit{e.g.}}

\newcommand{\ie}{\textit{i.e.}}
\newcommand{\iid}{i.i.d.}
\newcommand{\no}{\bot}
\newcommand{\naive}{na\"ive}

\newcommand{\vmax}{V_\mathrm{max}}
\newcommand{\hmtl}{\algfont{HMTL}}

\newcommand{\rmax}{\algfont{Rmax}}

\newcommand{\eee}{\algfont{E$^3$}}
\newcommand{\fe}{\algfont{ForcedExp}}
\newcommand{\ef}{\algfont{ExpFirst}}
\newcommand{\pacexp}{\algfont{PAC-Explore}}
\newcommand{\fmllrl}{\algfont{Finite-Model-RL}}

\newcommand{\E}{\mathbb{E}}
\renewcommand{\P}[1]{\mathbb{P}\left\{#1\right\}}

\newcommand{\norm}[1]{\left\|#1\right\|}
\newcommand{\setcard}[1]{\left|#1\right|}
\newcommand{\defeq}{\vcentcolon =}
\newcommand{\1}[1]{\mathbb{I}\left\{#1\right\}}

\newtheorem{theorem}{Theorem}
\newtheorem{lemma}[theorem]{Lemma}
\newtheorem{proposition}[theorem]{Proposition}

\theoremstyle{definition}

\newcommand{\mathcomment}[1]{\quad\text{(#1)}}

\definecolor{grey}{gray}{0.4}

\newcommand{\oldthm}[1]{\vspace{1mm}\noindent\textbf{Theorem~\ref{#1}.}}
\newcommand{\oldlem}[1]{\vspace{1mm}\noindent\textbf{Lemma~\ref{#1}.}}
\newcommand{\oldprop}[1]{\vspace{1mm}\noindent\textbf{Proposition~\ref{#1}.}}
\newcommand{\thmskip}{\vspace{2mm}}

\newcommand{\lihong}[1]{[[\textbf{LL:} \textit{#1}]]}
\newcommand{\emma}[1]{[[\textbf{EB:} \textit{#1}]]}
\renewcommand{\lihong}[1]{}
\renewcommand{\emma}[1]{}

\begin{document}
%
\newcommand{\textTitle}{The Online Coupon-Collector Problem and \\ Its Application to Lifelong Reinforcement Learning}
\title{\textTitle}
\author{Emma Brunskill\\
Department of Computer Science\\
Carnegie Mellon University\\
Pittsburgh, PA 15213 \\
\texttt{ebrunskill@cs.cmu.edu} \\
\And
Lihong Li \\
Microsoft Research \\
One Microsoft Way \\
Redmond, WA 98052 \\
\texttt{lihongli@microsoft.com}
}
\maketitle
\begin{abstract}
\begin{quote}
Transferring knowledge across a sequence of related tasks is an important challenge in reinforcement learning (RL).  Despite much encouraging empirical evidence, there has been little theoretical analysis.  In this paper, we study a class of lifelong RL problems: the agent solves a sequence of tasks modeled as finite Markov decision processes (MDPs), each of which is from a finite set of MDPs with the same state/action sets and different transition/reward functions.  Motivated by the need for cross-task exploration in lifelong learning, we formulate a novel \emph{online coupon-collector problem} and give an optimal algorithm.  This allows us to develop a new lifelong RL algorithm, whose overall sample complexity in a sequence of tasks is much smaller than single-task learning, even if the sequence of tasks is generated by an adversary.  Benefits of the algorithm are demonstrated in simulated problems, including a recently introduced human-robot interaction problem.
\end{quote}
\end{abstract}

\section{Introduction}
\label{sec:intro}
Transfer learning, the ability to take prior knowledge and use it to perform well on a new task, is an essential capability of intelligence. Tasks themselves often involve multiple steps of decision making under uncertainty.  Therefore, lifelong learning across multiple reinforcement-learning (RL)~\citep{Sutton98Reinforcement} tasks, or LLRL, is of significant interest. Potential applications are broad, from leveraging information across customers, to speeding robotic manipulation in new environments.  In the last decades, there has been much previous work on this problem, which predominantly focuses on providing promising empirical results but with little formal performance guarantees (\eg, \namecite{Ring97Child}, \namecite{wilson2007}, \namecite{Taylor09Transfer}, \namecite{Schmidhuber13Powerplay} and the many references therein), or in the offline/batch setting~\citep{Lazaric11Transfer}, or for multi-armed bandits~\citep{Azar13Sequential}.

In this paper, we focus on a special case of lifelong reinforcement learning which captures a class of interesting and challenging applications.  We assume that all tasks, modeled as finite Markov decision processes or MDPs, have the same state and action spaces, but may differ in their transition probabilities and reward functions.  Furthermore, the tasks are elements of a finite collection of MDPs that are initially unknown.\footnote{Given \emph{finite} sets of states and action, MDPs with similar transition/reward parameters have similar value functions.  Thus, finitely many policies suffice to represent near-optimal policies.} Such a setting is particularly motivated by applications to user personalization, in domains like education, health care and online marketing, where one can consider each ``task'' as interacting with one particular individual, and the goal is to leverage prior experience to improve performance with later users. Indeed, partitioning users into several groups with similar behavior has found uses in various application domains
\cite{Chu09Personalized,fern2014decision,liuEDM2015,nikolaidis2015efficient}: it offers a form of partial personalization, allowing the system to more quickly learn good interactions with the user (than learning for each user separately) but still offering much more personalization than modeling all individuals as the same.

A critical issue in transfer or lifelong learning is how and when to leverage information from previous tasks in solving the current one. If the new task represents a different MDP with a different optimal policy, then leveraging prior task information may actually result in substantially worse performance than learning with no prior information, a phenomenon known as \textit{negative transfer}~\cite{Taylor09Transfer}. Intuitively, this is partly because leveraging prior experience 
can prevent an agent from visiting states with different rewards in the new task, and yet would be visited under the optimal policy of the new task.  
In other words, 
 in lifelong RL, in addition to exploration typically needed to obtain optimal policies in single-task RL (\ie, single task exploration), 
the agent also needs sufficient exploration to 
uncover \emph{relations among tasks} (\ie, task-level transfer).

To this end, the agent faces an online discovery problem: the new task may be the same as one of prior tasks, or may be a novel one. The agent can treat it as a task that has been seen before (therefore transferring prior knowledge to solve it), or try to discover whether it is novel.  Failing to correctly treat a novel task as new, or treating an existing task as the same as a prior task, will lead to sub-optimal performance.


The main contributions are three-fold.  First, inspired by the need for online discovery in LLRL, we formulate and study a novel online coupon-collector problem (OCCP), providing algorithms with optimal regret guarantees.  These results are of independent interest, given the wide application of the classic coupon-collector problem.  Second, we propose a novel LLRL algorithm, which essentially is an OCCP algorithm that uses sample-efficient single-task RL algorithms as a black box.  When solving a sequence of tasks, compared to single-task RL, this LLRL algorithm is shown to have a substantially lower sample complexity of exploration, a theoretical measure of learning speed in online RL.  Finally, we provide simulation results on a simple gridworld simulation, 
and a simulated human-robot collaboration task recently introduced by \namecite{nikolaidis2015efficient}, in which there 
exist a finite set of different (latent) human user types with 
different preferences over their desired robot collaboration 
interaction. Our results illustrate the benefits and 
relative advantage of our new approach over prior ones.


\subsubsection{Related Work.}

There has been substantial interest in lifelong learning across sequential decision making tasks for decades; \eg, \namecite{Ring97Child}, \namecite{Schmidhuber13Powerplay}, and \namecite{White12Scaling}.  
Lifelong RL is closely related to transfer RL, in which information (or data) from source MDPs is used to accelerate learning in the target MDP~\cite{Taylor09Transfer}.  A distinctive element in lifelong RL is that every task is both a target and a source task.  Consequently, the agent has to explore the current task once in a while to allow better knowledge to be transferred to better solve future tasks---this is the motivation for the online coupon-collector problem we formulate and study here.

Our setting, of solving MDPs sampled from a finite set, is 
related to  
\namecite{konidaris2014hidden}'s  hidden parameter MDPs, which cover our setting 
and others where there is a latent variable that captures key aspects of a task.   \namecite{wilson2007} tackle a similar problem with a hierarchical Bayesian approach to modeling task-generation processes.
Most prior work on lifelong/transfer RL has focused on algorithmic and empirical innovations, with little theoretical analysis for \emph{online} RL.  An exception is a two-phase algorithm~\cite{Brunskill13Sample}, which has provably lower sample complexity than single-task RL, but makes a few critical assumptions.
Our setting is more general: tasks may be selected adversarially, instead of stochastically~\cite{wilson2007,Brunskill13Sample}.  
Consequently, we do not assume 
a minimum task sampling probability, or knowledge of 
the cardinality of the (latent) set of MDPs.  
This allows our algorithm to be applied in more realistic problems such as personalization domains where the number of user ``types'' is typically unknown in advance.
In addition, ~\namecite{Bouammar15Safe} recently introduced and provided regret 
bounds (as a function of the number of tasks) of a policy-search algorithm 
for LLRL.  Each task's policy parameter is represented as a linear combination of shared latent variables, allowing it to be used in continuous domains.
However, 
in addition to local optimality guarantees typical in policy-search methods, lack of sufficient exploration in their approach may also lead to suboptimal policies.

In addition to the original coupon-collector problem,
to be described in the next section, our online coupon-collector problem is related to 
bandit problems~\citep{Bubeck12Regret} that also require efficient exploration. 
In bandits every action leads to an observed loss, while in OCCP only one action has observable loss.  Apple tasting~\citep{Helmbold00Apple} has a similar flavor as OCCP, but with a different structure in the loss matrix; furthermore, its analysis is in the mistake-bound model that is not suitable here. 
\namecite{Langford02Competitive} study an abstract model for exploration, 
but their setting assumes a non-decreasing, deterministic reward sequence, while we allow non-monotonic and stochastic (or even adversarial) reward sequences. 
Consequently, an explore-first strategy is optimal in their setting but not in OCCP.  
Furthermore, they analyze competitive ratios, while we focus on excessive loss.  
\namecite{Bubeck14Optimal} tackle a very different problem called ``optimal discovery'', for quick identification of hidden elements assuming access to different sampling distributions.  Finally, compared to the missing mass problem~\cite{McAllester00Convergence}, which is about pure predictions, OCCP involves decision making, thus requires balancing exploration and exploitation.

\section{The Online Coupon-Collector Problem}
\label{sec:odp}

Motivated by the need for cross-task exploration to discover novel MDPs in LLRL, we formulate and study a novel problem that is an online version of the classic Coupon-Collector Problem, or CCP~\cite{Vonschelling54Coupon}.  Solutions to online CCP play a crucial role in developing a new lifelong RL algorithm in the next section.  Moreover, the problem may be of independent interest in many disciplines 
like optimization, biology, communications, and cache management in operating systems, where CCP has found important applications~\cite{Boneh97Coupon,Berenbrink09Weighted}, as well as in other meta-learning problems that require efficient exploration to uncover cross-task relation.

\subsection{Formulation}

In the Coupon-Collector Problem, there is a multinomial distribution $\mu$ over a set $\Mset$ of $C$ coupon types.  In each round, one type is sampled from $\mu$.  
Much research has been done to study probabilistic properties of the 
(random) time when all $C$ coupons are first collected, especially its expectation (\eg, \namecite{Berenbrink09Weighted} and references therein).

In our \emph{Online Coupon-Collector Problem} or OCCP, $C=\setcard{\Mset}$ is unknown.  Given a coupon, the learner may probe the type or skip; thus, $\Aset=\{\actProbe~\text{(``probe'')}, \actSkip~\text{(``skip'')}\}$ is the binary action set.  
The learner is also given four constants, $\rho_0 < \rho_1 \le \rho_2 < \rho_3$, specifying the loss matrix $L$ in \tblref{tbl:loss}.

\begin{table}[h]
  \centering
  \caption{OCCP loss matrix: rows indicate actions; columns indicate whether the current item is novel or not.  The known constants, 
$\rho_0 < \rho_1 \le \rho_2 < \rho_3$,
  specify costs of actions in different situations.} \label{tbl:loss}
\begin{tabular}{l|cc}
 & $\1{M_t\in\Mset_t}$ & $\1{M_t\notin\Mset_t}$ \\
\hline
$\actSkip$ & $\rho_0$ & $\rho_3$ \\
$\actProbe$ & $\rho_1$ & $\rho_2$
\end{tabular}
\end{table} 

The game proceeds as follows.  Initially, the set of discovered items $\Mset_1$ is $\emptyset$.  For round $t=1,2,\ldots,T$:
\begin{compactitem}
\item{Environment selects a coupon $M_t\in\Mset$ of unknown type.}
\item{The learner chooses action $A_t\in\Aset$, and suffers loss $L_t$ as specified in the loss matrix of \tblref{tbl:loss}.  The learner observes $L_t$ if $A_t=\actProbe$, and $\no$ (``no observation'') otherwise.}
\item{If $A_t=\actProbe$, $\Mset_{t+1} \leftarrow \Mset_t \cup \{M_t\}$; else $\Mset_{t+1}\leftarrow\Mset_t$.}
\end{compactitem}

At the beginning of round $t$, define the \emph{history} up to $t$ as $H_t \defeq (\Mset_1,A_1,L_1,\Mset_2,A_2,L_2,\ldots,\Mset_{t-1},A_{t-1},L_{t-1})$.  An algorithm is \emph{admissible}, if it chooses actions $A_t$ based on $H_t$ and possibly an external source of randomness.
We distinguish two settings.  In the \emph{stochastic} setting, environment samples $M_t$ from an unknown distribution $\mu$ over $\Mset$ in an \iid\ (independent and identically distributed) fashion.  In the \emph{adversarial} setting, the sequence $(M_t)_t$ can be generated by an adversarial in an arbitrary way that depends on $H_t$.

If the learner \emph{knew} the type of $M_t$, the optimal strategy would be to choose $A_t=\actProbe$ if $M_t\notin\Mset_t$, and $A_t=\actSkip$ otherwise.  The loss is $\rho_2$ if $M_t$ is a new type, and $\rho_0$ otherwise.  Hence, after $T$ rounds, if $C^*\le C$ is the number of distinct items in the sequence $(M_t)_t$, this ideal strategy has the loss:
\begin{equation}
L^*(T) \defeq \rho_2C^*+\rho_0(T-C^*)\,. \label{eqn:optimal-loss}
\end{equation}
The challenge, of course, is that the learner does not know $M_t$'s type before choosing $A_t$.  She thus has to balance exploration (taking $A_t=\actProbe$ to see if $M_t$ is novel) and exploitation (taking $A_t=\actSkip$ to yield small loss $\rho_0$ if it is likely that $M_t\in\Mset_t$).  Clearly, over- and under-exploration result in suboptimal strategies.  We are therefore interested in finding algorithms $\Aalg$ to have smallest cumulative loss as possible.

Formally, an OCCP algorithm $\Aalg$ is a possibly stochastic function that maps histories to actions: $A_t=\Aalg(H_t)$.  The total $T$-round loss suffered by $\Aalg$ is $L(\Aalg,T) \defeq \sum_{t=1}^T L_t$.  The \emph{$T$-round regret} of an algorithm $\Aalg$ is $R(\Aalg,T) \defeq L(\Aalg,T) - L^*(T)$, and its expectation by $\bar{R}(\Aalg,T)\defeq\E[R(\Aalg,T)]$, where the expectation is taken with respect to any randomness in the environment as well as in $\Aalg$.


%


\subsection{Explore-First Strategy}

In the \emph{stochastic} case,
it can be shown that if an algorithm chooses $\actProbe$ for a total of $E$ times, its expected regret is smallest if these actions are chosen at the very beginning.  The resulting strategy is sometimes called \algfont{explore-first}, or \ef\ for short, in the multi-armed bandit literature.

With knowledge of $\mu_m\defeq\min_{M\in\Mset}\mu(M)$, one may set $E$ so that all types in $\Mset$ will be discovered in the first (probing) phase consisting of $E$ rounds with high probability.
This results in a high-probability regret bound, which can be used to establish an expected regret bound, as summarized below.
A proof is given in \appref{sec:ef-lemma}.
\newcommand{\textPropEf}{
For any $\delta\in(0,1)$, let $E=\mu_m^{-1}\ln\frac{1}{\mu_m\delta}$ where $\mu_m=\min_{M\in\Mset}\mu(M)$.  Then, with probability $1-\delta$, $R(\ef,T) \le 
\frac{\rho_1-\rho_0}{\mu_m}\ln\frac{1}{\mu_m\delta}$.  Moreover, 
if $E=\frac{1}{\mu_m}\ln\frac{(\rho_3-\rho_0)T}{\rho_1-\rho_0}$,
then the expected regret is $\bar{R}(\ef,T) \le \frac{\rho_1-\rho_0}{\mu_m}\left(\ln\frac{(\rho_3-\rho_0)T}{\rho_1-\rho_0}+1\right)$.
}
\begin{proposition} \label{prop:ef}
\textPropEf
\end{proposition}

\subsection{Forced-Exploration Strategy}

While \ef\ is effective in stochastic OCCP, it requires to know $\mu_m$, and the probing phase may be too long for small $\mu_m$.  Moreover, in many scenarios, the sampling process may be non-stationary (\eg, different types of users may use the Internet at different time of the day) or even adversarial (\eg, an attacker may present certain MDPs in earlier tasks in LLRL to cause an algorithm to perform poorly in future ones).
We now study a more general algorithm, \fe, based on forced exploration, and prove a regret upper bound.  The next subsection will present a matching lower bound, indicating the algorithm's optimality.

Before the game starts, the algorithm chooses a \emph{fixed} sequence of ``probing rates'': $\eta_1,\ldots,\eta_T\in[0,1]$.  In round $t$, it chooses actions accordingly: $\P{A_t=\actSkip}=1-\eta_t$ and $\P{A_t=\actProbe}=\eta_t$.  The main result in this subsection is as following, proved in \appref{sec:fe-proofs}.
\newcommand{\textThmFeMain}{%
Let $\eta_t = t^{-\alpha}$ (polynomial decaying rate) for some parameter $\alpha\in(0,1)$.  Then, for any given $\delta\in(0,1)$,
\begin{equation}
R(\fe,T) \le C^* \rho_3 \left(T^\alpha\ln\frac{C^*}{\delta}+1\right)\,, \label{eqn:fe-hp-regret}
\end{equation}
with probability $1-\delta$.  The expected regret is 
$\bar{R}(\fe,T) \le C^*\rho_3 T^{\alpha} + \frac{\rho_1}{1-\alpha} T^{1-\alpha}$.
Both bounds are $O(\sqrt{T})$ by by choosing $\alpha=1/2$.
}
\begin{theorem} \label{thm:fe-main}
\textThmFeMain
\end{theorem}

The results show that \fe\ eventually performs as well as the \emph{hypothetical} optimal strategy that knows the type of $M_t$ in every round $t$, no matter how $M_t$ is generated.  Moreover, the per-round regret decays on the order of $1/\sqrt{T}$, which we will show to be optimal shortly.


\subsection{Lower Bounds}

The main result in this subsection, \thmref{thm:lb}, shows the $O(\sqrt{T})$ regret bound for \fe\ is essentially not improvable, in term of $T$-dependence, even in the stochastic case.  The idea of the proof, given in \appref{sec:proof-lb}, is to construct a hard instance of stochastic OCCP.  On one hand, $\Omega(\sqrt{T})$ regret is suffered unless all $C$ types are discovered.  On the other hand, most of the types have small probability $\mu_m$ of being sampled, requiring the learner to take the exploration action $\actProbe$ many times to discover all $C$ types.  The lower bound follows from an appropriate value of $\mu_m$.
\newcommand{\textThmLb}{
There exists an OCCP where every admissible algorithm has an expected regret of $\Omega(\sqrt{T})$, and for sufficiently small $\delta $, the regret is $\Omega(\sqrt{T})$ with probability $1-\delta$. 
}
\begin{theorem} \label{thm:lb}
\textThmLb
\end{theorem}

Note our goal here is to find a matching lower bound in terms of $T$.  We do not attempt to match dependence on other quantities like $C$, which are often less important than $T$.

The lower bound may seem to contradict $\ef$'s \emph{logarithmic} upper bound in \propref{prop:ef}.  However, that upper bound is problem specific and requires knowledge of $\mu_m$.  Without knowing $\mu_m$, the algorithm has to choose $\mu_m=\Theta(\frac{1}{\sqrt{T}})$ in the probing phrase; otherwise, there is a chance it may not be able to discover a type $M$ with $\mu(M)=\Omega(\frac{1}{\sqrt{T}})$, suffering $\Omega(\sqrt{T})$ regret.  With this value of $\mu_m$, the bound in \propref{prop:ef} has an $\tilde{O}(\sqrt{T})$ dependence.

\section{Application to PAC-MDP Lifelong RL} 
\label{sec:general}

Building on the OCCP results established in the previous section, we now turn to lifelong RL.

\subsection{Preliminaries}
\label{sec:prelim}

We consider RL~\cite{Sutton98Reinforcement} in discrete-time, finite MDPs specified by a five-tuple: $\langle\Sset,\Aset,P,R,\gamma\rangle$, where $\Sset$ is the set of states ($S\defeq\setcard{\Sset}$), $\Aset$ the set of actions ($A\defeq\setcard{\Aset}$), $P$ the transition probability function, $R:\Sset\times\Aset\to[0,1]$ the reward function, and $\gamma\in(0,1)$ the discount factor.  
Initially, $P$ and $R$ are unknown.  Given a policy $\pi:\Sset\to\Aset$, its state and state--action value functions are denoted by $V^\pi(s)$ and $Q^\pi(s,a)$, respectively.  The optimal value functions are $V^*$ and $Q^*$.  Finally, let $\vmax$ be a known upper bound of $V^*(s)$, which is at most $1/(1-\gamma)$ but can be much smaller.

Various frameworks have been studied to capture the learning speed of single-task online RL algorithms, such as regret analysis~\cite{Jaksch10Near}.  Here, we focus on another useful notion known as \emph{sample complexity of exploration}~\citep{Kakade03Sample}, or \emph{sample complexity} for short.  Some of our results, especially those related to cross-task exploration and OCCP, may also find use in regret analysis. 

Any RL algorithm $\Aalg$ can be viewed as a nonstationary policy, whose value functions, $V^\Aalg$ and $Q^\Aalg$, are defined similarly to the stationary-policy case.  When $\Aalg$ is run on an unknown MDP, we call it a mistake at step $t$ if the algorithm chooses a suboptimal action, namely, $V^*(s_t)-V^{\Aalg_t}(s_t)>\epsilon$.  
We define the sample complexity of $\Aalg$, $\zeta(\epsilon,\delta)$ as the 
maximum number of mistakes, with probability at least $\delta$. 
If $\zeta$ is polynomial in $S$, $A$, $1/(1-\gamma)$, $1/\epsilon$, and $\ln(1/\delta)$, then $\Aalg$ is called \emph{PAC-MDP}~\citep{Strehl09Reinforcement}.  

Most PAC-MDP algorithms~\cite{Kearns02Near,Brafman02Rmax,Strehl09Reinforcement}
 work by assigning maximum reward to state--action pairs 
that have not been visited often enough to obtain reliable transition/reward parameters.  
The \fmllrl\ algorithm used for LLRL~\cite{Brunskill13Sample} 
leverages a similar idea, where the current RL task is close to one of a finite set of known MDP models.

\subsection{Cross-task Exploration in Lifelong RL}

In lifelong RL, the agent seeks to maximize total reward as it acts in a sequence of $T$  tasks.  If the tasks are related, learning speed is expected to improve by transferring knowledge obtained from prior tasks.  Following previous work~\citep{wilson2007,Brunskill13Sample}, and motivated by many applications~\cite{Chu09Personalized,fern2014decision,nikolaidis2015efficient,liuEDM2015}, we assume a finite set $\Mset$ of possible MDPs.  The agent solves a sequence of $T$ tasks, with $M_t\in\Mset$ denoting the (unknown) MDP of task $t$.  Before solving the task, the agent does not know whether or not $M_t$ has been encountered before.  It then acts in $M_t$ for $H$ steps, where $H$ is given, and can take advantage of any information extracted from solving prior tasks $\{M_1,\ldots,M_{t-1}\}$.  
%
Our setting is more general, allowing tasks to be chosen \emph{adversarially}, in contrast to prior work that focused on the stochastic case~\cite{wilson2007,Brunskill13Sample}.

\begin{algorithm}[t]
\caption{Lifelong RL based on \fe}
\begin{algorithmic}[1] \label{alg:eee-fe}
\STATE \textbf{Input}: $\alpha\in(0,1)$, $m\in\Nset$, $L\in\Nset$
\STATE Initialize $\hat{\Mset}\leftarrow\emptyset$
\FOR{$t=1,2,\ldots$}
\STATE Generate a random number $\xi\sim \mathrm{Uniform}(0,1)$
\IF{$\xi<t^{-\alpha}$ \hspace{3mm}{\color{grey}{\small (probing to discover new MDP)\hspace{3mm}}}}
\STATE Run \pacexp\ with parameters $m$ and $L$ to fully explore all states in $M_t$, so that every action is taken in every state for at least $m$ times. \label{alg:eee-fe:explore:start}
\STATE After \pacexp\ finishes, choose actions by an optimal policy of the empirical model $\hat{M}_t$.
\IF{for all existing models $\hat{M}\in\hat{\Mset}$, $\hat{M}_t$ has a non-overlapping confidence intervals in some state--action pair's transition/reward parameters}
\STATE $\hat{\Mset} \leftarrow \hat{\Mset} \cup \{\hat{M}_t\}$ \label{alg:eee-fe:explore:end}
\ENDIF
\ELSE
\STATE Run \fmllrl\ with $\hat{\Mset}$ \label{alg:eee-fe:exploit}
\ENDIF
\ENDFOR
\end{algorithmic}
\end{algorithm}

In comparison to single-task RL, performing additional exploration 
in a task (potentially beyond that needed for reward maximization in the 
current task), may be advantageous in the LLRL setting, since 
such information may help the agent perform better in future tasks. 
Indeed, prior work~\cite{Brunskill13Sample} has demonstrated that learning 
the latent structure of the possible MDPs that may be encountered 
can lead to significant reductions in the sample complexity in 
later tasks. We can realize this benefit by explicitly identifying this 
latent shared structure. 

This observation inspired our abstraction of OCCP, which we now formalize its relation to LLRL. 
Here, the probing action ($\actProbe$) corresponds to doing full exploration in the current task, while the skipping action ($\actSkip$) corresponds to applying transferred knowledge to accelerate learning. We use our OCCP \fe\ algorithm resulting in \algref{alg:eee-fe}; overloading terminology, we refer to this LLRL algorithm as \fe.  In contrast, the two-phase LLRL algorithm of \namecite{Brunskill13Sample} essentially uses \ef\ to discover new MDPs, and is referred to as \ef.

At round $t$, if probing is to happen, \fe\ performs \pacexp~\cite{Guo15Concurrent}, outlined in \algref{alg:pacexp} of \appref{app:alg}, to do full exploration of $M_t$ to get an accurate empirical model $\hat{M}_t$.  To determine whether $M_t$ is new, the algorithm checks if $\hat{M}_t$'s parameters' confidence intervals are disjoint from every $\hat{M}\in\hat{\Mset}$ in at least one state--action pair.  If so, we add $\hat{M}_t$ to the set $\hat{\Mset}$.


If probing is \emph{not} to happen, 
the agent assumes $M_t\in\hat{\Mset}$, and follows the \fmllrl\ algorithm~\cite{Brunskill13Sample}, which is an extension of \rmax\ to work with finitely many MDP models.  With \fmllrl, the amount of exploration scales with the number of models, rather than the number of state--action pairs.  Therefore, the algorithm gains in sample complexity by reducing unnecessary exploration from transferring prior knowledge, \emph{if} the current task is already in $\hat{\Mset}$.

Note that \algref{alg:eee-fe} is a meta-algorithm, where single-task-RL components like \pacexp\ and \fmllrl\ may be replaced by similar algorithms.



\paragraph{Remark.} \fe\ may appear \naive\ or simplistic, as it decides whether to probe a 
new task before seeing any data in $M_t$.  
It is easy to allow the algorithm to switch 
from non-probing ($\actSkip$) to probing ($\actProbe$) while acting in 
$M_t$, whenever $M_t$ appears different from all MDPs in $\hat{\Mset}$ (again, by comparing confidence intervals of model parameters).  Although this change can be beneficial in practice, it does \emph{not} improve worst-case sample complexity: if we are in the non-probing case running \fmllrl\ in a MDP not in $\hat{\Mset}$, there is no guarantee to identify the current task as a new one.
This is because by assuming that the 
current MDP is one of the models in $\hat{\Mset}$, the learner may follow a policy that never sufficiently explores  \emph{informative} state--action pair(s) that could have revealed the current MDP is novel.
Therefore, from a theoretical (worst-case) perspective, it is not critical to allow the algorithm to switch to the probing mode.
 
Similarly, switching from probing to non-probing in the middle of a task is in general not helpful, as shown in the following example.  
Let $\Sset=\{s\}$ contain a single state, so $P(s|s,a)\equiv1$ and MDPs in $\Mset$ differ only in the reward function.  Suppose at round $t$, the learner has discovered a set of MDPs $\hat{\Mset}$ from the past, and chooses to probe, thus running \pacexp.  After some steps in $M_t$, if the learner switches to non-probing before trying every action $m$ times in all states, there is a risk of \emph{under-exploration}: $M_t$ may be a new MDP not in $\hat{\Mset}$; it has the same rewards on optimal actions for some $M\in\hat{\Mset}$, but has even higher reward for another action that is not optimal for any $M'\in\hat{\Mset}$.
By terminating exploration too early, the learner may fail to identify the optimal action in $M_t$, ending up with a poor policy.

\subsection{Sample-Complexity Analysis}
\label{sec:general-theory}

This section gives a sample-complexity analysis for \algref{alg:eee-fe}.  For convenience, we use $\theta_M$ to denote the dynamics of an MDP $M\in\Mset$: for each $(s,a)$, $\theta_M(\cdot|s,a)$ is an $(S+1)$-dimensional vector, with the first $S$ components giving the transition probabilities to corresponding next states, $P(s'|s,a)$, and the last component the average immediate reward, $R(s,a)$.  The model difference in $(s,a)$ between $M$ and $M'$, denoted $\|\theta_{M}(\cdot|s,a)-\theta_{M'}(\cdot|s,a)\|$, is the $\ell_2$-distance between the two vectors.  
Finally, we let $N$ be an upper bound on the number of next states in the transition models in all MDPs $M\in\Mset$; note that $N$ is no larger than $S$ but can be much smaller in many problems.

The following assumptions are made in the analysis:
\begin{compactenum}
\item{There exists a known quantity $\Gamma>0$ such that for every two distinct MDPs $M, M'\in\Mset$, there exists some $(s,a)$ so that $\|\theta_M(\cdot|s,a)-\theta_{M'}(\cdot|s,a)\|>\Gamma$;}
\item{There is a known diameter $D$, such that: for any $M\in\Mset$, any states $s$ and $s'$, there is a policy $\pi$ that takes an agent to navigate from $s$ to $s'$ in at most $D$ steps \emph{on average};}
\item{There are $H \ge H_0$ steps to solve each task $M_t$, where $H_0=O\left(SAN\log\frac{SAT}{\delta}\max\{\Gamma^{-2},D^2\}\right)$.}
\end{compactenum}
The first assumption requires two distinct MDPs differ by a sufficient amount in their dynamics in at least one state--action pair, and is made for convenience to encode prior knowledge about $\Gamma$.  Note that if $\Gamma$ is not known beforehand, one can set $\Gamma$ to $\Gamma_0=O\big(\epsilon(1-\gamma)/(\sqrt{N}\vmax)\big)$: if two MDPs differ by no more than $\Gamma_0$ in every state--action pair, an $\epsilon$-optimal policy in one MDP will be an $O(\epsilon)$-optimal policy in another.  The second and third assumptions are the major ones needed in our analysis.  The diameter $D$, 
introduced by \namecite{Jaksch10Near}, is typically not needed in \emph{single-task} sample-complexity analysis, but it seems nontrivial to avoid in a \emph{lifelong} learning setting.
Without the diameter or the long-horizon assumption, a learner can get stuck in a subset of states that prevent it from identifying the current MDP.  In such situations, it is unclear how the learner can \emph{reliably} transfer knowledge to better solve future tasks.

With these assumptions, the main result is as follows. Note that it is possible to use refined single-task analysis such as \namecite{Lattimore12Pac} to get better constants for $\rho_0$ and $\rho_3$ below.  We defer that to future work, and instead focus on showing the benefits of lifelong learning.
%
\newcommand{\textThmLlrlfe}{%
Let \algref{alg:eee-fe} with proper choices of parameters be run on a sequence of $T$ tasks, each from a set $\Mset$ of $C$ MDPs. Then, with prob. $1-\delta$, the number of steps in which the algorithm is not $\epsilon$-optimal \emph{across all $T$ tasks} is $\tilde{O}\big( \rho_0 T + C\rho_3\sqrt{T}\ln\frac{C}{\delta} \big)$, where $\rho_0=CD/\Gamma^2$ and $\rho_3=H$.
}
\begin{theorem} \label{thm:llrl-fe}
\textThmLlrlfe
\end{theorem}

While single-task RL typically has a per-task sample complexity $\zeta_s$ that at least scales linearly with $SA$, \algref{alg:eee-fe} converges to a per-task sample complexity of $\tilde{O}(\rho_0)$, which is often much lower.  Furthermore, a bound on the \emph{expected} sample complexity can be obtained in a similar way, by the corresponding expected-regret bound in \thmref{thm:fe-main}. Intuitively, in the OCCP setting, we quantified the loss (equivalently, regret); in LLRL, the loss corresponds to number of non-$\epsilon$-optimal steps, and so a loss bound translates directly into a sample-complexity bound.
%
%


The proof (\appref{sec:llrl-proofs}) proceeds by analyzing the sample complexity bounds for all four possible cases (corresponding to the four entries in the OCCP loss matrix in \tblref{tbl:loss}) when solving the $M_t$, and then combining them with \thmref{thm:fe-main} to yield the desired results.  
A key step is to ensure that when probing happens, the type of $M_t$ will be discovered successfully with high probability.  This is achieved by a couple of key technical lemmas below, which also elucidate where our assumptions are used in the analysis.  

The first lemma ensures all state--actions can be visited sufficiently often in finite steps, when the MDP has a small diameter.  
For convenience, define $H_0(m)\defeq O(SADm)$.
\newcommand{\textLemDiameter}{%
For a given MDP, \pacexp\ with input $m \ge m_0$ and $L=3D$ will visit all state--action pairs at least $m$ times in no more than $H_0(m)$ steps with probability $1-\delta$, where $m_0=O\left(ND^2\log\frac{N}{\delta}\right)$ is some constant.
}
\begin{lemma} \label{lem:diameter}
\textLemDiameter
\end{lemma}

The second lemma establishes the fact that when \pacexp\ is run on a sequence of $T$ tasks, with high probability, it successfully infers whether $M_t$ has been included in $\hat{\Mset}$, for every $t$.  This result is a consequence of Lemma~\ref{lem:diameter} and the assumption involving $\Gamma$.
%
\newcommand{\textLemCorrectDiscovery}{%
With input parameters $H\ge H_0(m)$ and $m=72N\log\frac{4SAT}{\delta}\max\{\Gamma^{-2},D^2\}$ in \algref{alg:eee-fe}, the following holds with probability $1-2\delta$: for every task in the sequence, the algorithm detects it is a new task if and only if the corresponding MDP has not been seen before.
}
\begin{lemma} \label{lem:correct-discovery}
\textLemCorrectDiscovery
\end{lemma}

\begin{figure}[t]
\centering
\includegraphics[width=0.75\columnwidth]{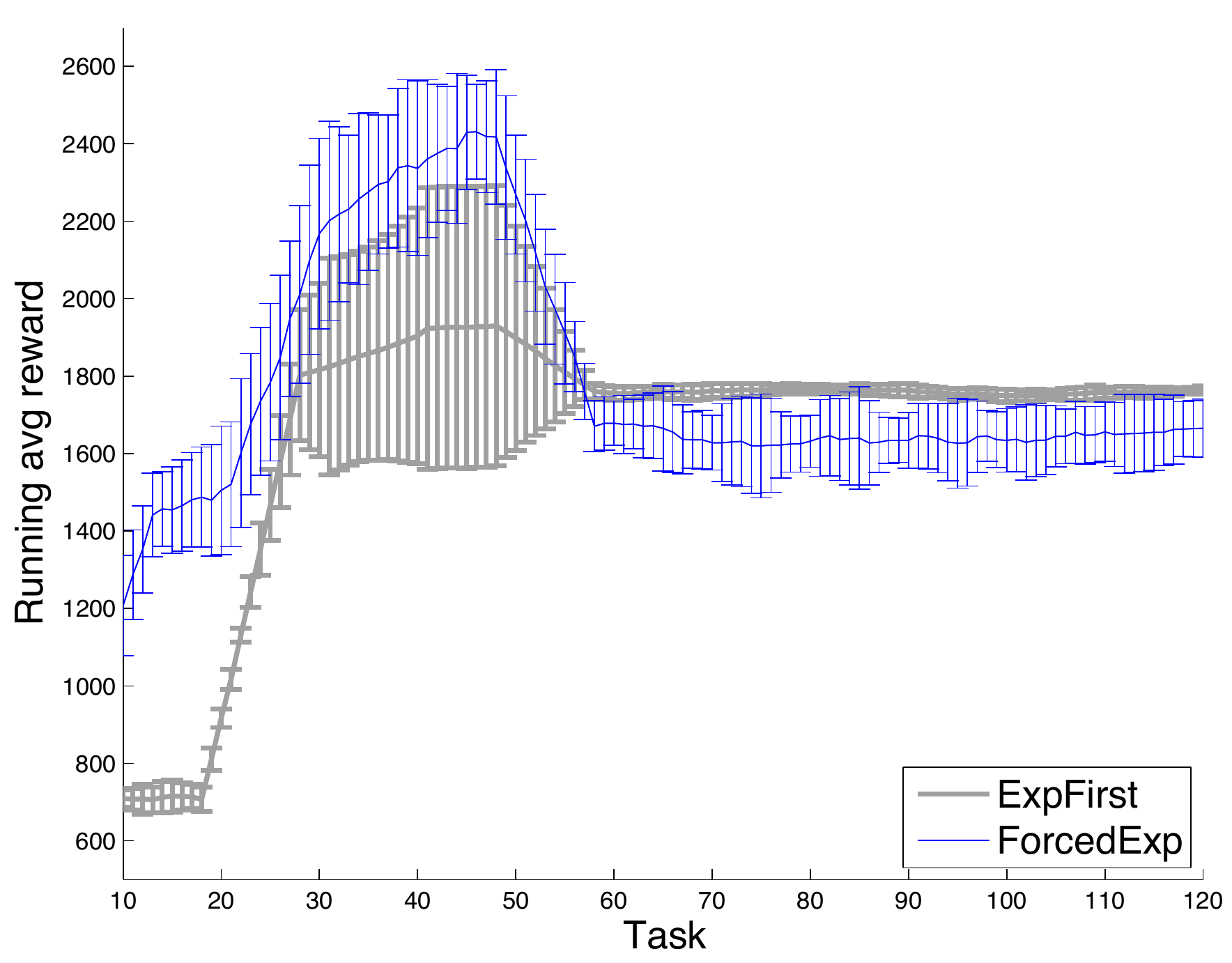} \vspace{-4mm}
\caption{Gridworld: nonstationary task selection. $10$-task smoothed running average of reward per task with 1 std error bars.} \label{fig:adversary} 
\vspace{-4mm}
\end{figure}

\section{Experiments}
\label{sec:exp}

Our simulation results 
illustrate that our lifelong RL setting can capture 
interesting domains, and to demonstrate the benefit of our introduced approach over a prior algorithm with formal sample-complexity guarantees~\cite{Brunskill13Sample} that is based on \ef. 
Due to space limitations, full details are provided in \appref{sec:exp_det}. \\
\noindent\textbf{Gridworld}.
We first consider a simple $5$ by $5$ stochastic gridworld domain with $4$ distinct MDPs to illustrate the salient properties of \fe.  
In each of the $4$ MDPs one corner offers high reward (sampled from a Bernoulli with parameter $0.75$) and all other rewards are $0$. In MDP 4 both the same corner as MDP 3 is rewarding, and the opposite corner is a Bernoulli with parameters $0.99$.  

In the stochastic setting when all tasks are sampled with equal probability, we compared \ef, \fe\ and \hmtl---a Bayesian hierarchical multi-task RL algorithm~\cite{wilson2007}.  As expected, all approaches did well in this setting. We next focus on comparing \ef\ and \fe which have finite sample guarantees.

We first consider tasks sampled from \emph{nonstationary} distributions.  
Across $100$ tasks all $4$ MDPs have identical frequencies, but an adversary chooses to only select from MDPs 1--3 during the first (probing-only) phrase of \ef\, 
before switching MDP 4 for $25$ tasks, and then switching back to randomly selecting the first three MDPs. 
MDP 4 can obtain similar rewards as MDP 1 using the same policy as for MDP 1, but can obtain higher rewards if the agent explicitly explores to discover the state with higher reward. \fe\ will randomly probe MDP 4, thus identifying this new optimal policy, which is why it eventually picks up the new MDP and obtains higher reward (See Figure~\ref{fig:adversary}).. \ef\ sometimes successfully infers the task belongs to a new MDP, but only if it happens to encounter 
the state that distinguished MDPs 1 and 4. This illustrates the benefit of continued active exploration in nonstationary or adversarial settings. 

\noindent\textbf{Simulated Human-Robot Collaboration}. 
We next consider a more interesting human-robot collaboration problem studied by  
\namecite{nikolaidis2015efficient}.
In this work, the authors learned $4$ models of user types 
based on prior data collected about a paired interaction task 
in which a human collaborates with a robot to paint a box. 
Using these types as a latent state in a 
mixed-observability MDP enabled significant improvements 
over not modeling such types in an experiment with real human 
robot collaborations. 
 
In our LLRL simulation each task was randomly sampled from the $4$ MDP models learned by \namecite{nikolaidis2015efficient}.
This domain was much larger than our grid world environment, involving $605$ states and $27$ actions.
It is typical in such personalization problems that not all user types have the same frequency.
Here, we chose the sampling distribution $\mu=(0.07,0.31,0.31,0.31)$.
The length of \ef's initial proving period is 
dominated by $\frac{1}{\mu_m}=\frac{1}{0.07}$. 
Experiments were repeated $30$ runs, each consisting of $80$ tasks.

The long probing phase of \ef\ is costly, especially if the total number of tasks is small, since too much time is spent on discovering new MDPs.
This is shown in Table 2, where our 
 \fe\ demonstrates a significant advantage by leveraging past 
experience much earlier than \ef, leading to significantly 
higher reward both during phase 1 and overall (Mann-Whitney 
U test, $p<0.001$ in both cases).  
Of course, eventually \ef\ will exhibit near-optimal performance 
in its second (non-probing) phase,
whereas \fe\ will continue probing with diminishing probability.
However, \fe\ can exhibit substantial jump-start benefit 
when the underlying MDPs are drawn from a stationary but nonuniform distribution.

\begin{table}[t] 
\small
\caption{Average per-task reward (and std. deviation) in each phase and overall.  Gains with statistical significance are highlighted.} 
\label{tab:hri_phase}
\begin{center} 
{\scriptsize
\vspace{-4mm}
\begin{tabular}{|c||c|c|c|}
\hline
& Phase 1 & Phase 2 & Overall ($80$ tasks) \\ \hline\hline
\ef & $18305 (1609)$ & $19428 (1960)$ & $18543 (1683)$ \\ \hline
\fe & $\textbf{18745 (482)}$ & $19012 (1904)$ & $\textbf{18801 (1923)}$ \\ \hline
\end{tabular}
}
\vspace{-4mm}
\end{center}
\end{table}

These results suggest \fe\ achieves comparabe or 
substantially better performance than prior methods, 
especially in nonuniform or nonstationary LLRL problems. 

\section{Conclusions}
\label{sec:conclude}

In this paper, we consider a class of lifelong RL problems that capture a broad range of interesting applications.  Our work emphasizes the need for efficient cross-task exploration that is unique in lifelong learning.  This led to a novel online coupon-collector problem, for which we give optimal algorithms with matching upper and lower regret bounds.  With this tool, we develop a new lifelong RL algorithm, and analyze its total sample complexity across a sequence of tasks.  Our theory quantifies how much gain is obtained by lifelong learning, compared to single-task learning, even if the tasks are adversarially generated.  The algorithm was empirically evaluated in two simulated problems, including a simulated human-robot collaboration task, demonstrating its relative strengths compared to prior work.

In the future, we are interested in extending our work to LLRL with continuous MDPs.  It is also interesting to investigate the empirical and theoretical properties of Bayesian approaches, such as Thompson sampling~\cite{osband2013more}, in lifelong RL.  These algorithms allow rich information to be encoded into a prior distribution, and empirically are often effective at taking advantage of such prior information.

\newpage

\bibliographystyle{aaai}
\bibliography{refs_short}

\newpage

\setcounter{secnumdepth}{2}
\appendix

\section{Proof for \propref{prop:ef}}
\label{sec:ef-lemma}

For convenience, statements of theorems, lemmas and propositions from the main text will be repeated when they are proved in the appendix.

\oldprop{prop:ef}
\textit{\textPropEf}

\begin{proof}
We start with the high-probability bound.  Fix any $M\in\Mset$.  The probability that it is \emph{not} sampled in the first $E$ rounds can be bounded as follows:
\begin{eqnarray}
  \lefteqn{\P{M\notin\{M_1,\ldots,M_E\}}} \nonumber \\
 &=& (1-\mu(M))^E \nonumber \\
 &\le& \exp(-\mu(M)E) \mathcomment{by inequality $1-x \le e^{-x}$} \nonumber \\
 &\le& \exp(\ln(\delta\mu_m)) \mathcomment{by definition, $\mu_m\le\mu(M)$} \nonumber \\
 &=& \mu_m\delta\,. \label{eqn:missing-mass-prob}
\end{eqnarray}
Consequently, we have
\[
\P{\exists M\in\Mset, M\notin\{M_1,\ldots,M_E\}} \le C\delta\mu_m \le \delta\,,
\]
where the first inequality is due to \eqnref{eqn:missing-mass-prob} and a union bound applied to all $M\in\Mset$, and the second inequality follows from the observation that $C \le 1/\mu_m$.

We have thus proved that, with probability at least $1-\delta$, all types in $\Mset$ will be sampled at least once in the first $E$ rounds, and \ef\ will have the minimal loss $\rho_0$ for all $t>E$.  Thus, with probability $1-\delta$, we have
\begin{equation}
L(\ef,T) = \rho_2C^* + \rho_1(E-C^*) + \rho_0(T-E)\,, \label{eqn:ef-hp-loss}
\end{equation}
where the first two terms correspond to loss incurred in the first $E$ rounds, and the last term corresponds to loss incurred in the remaining $T-E$ rounds.
Subtracting the optimal loss of \eqnref{eqn:optimal-loss} from \eqnref{eqn:ef-hp-loss} above gives the desired high-probability regret bound:
\begin{eqnarray}
R(\ef,T) &=& (\rho_1-\rho_0)(E-C^*) \label{eqn:ef-hp-regret} \\
  &\le& (\rho_1-\rho_0)E\,. \nonumber
\end{eqnarray}

We now prove the expected regret bound.  Since \eqnref{eqn:ef-hp-regret} holds with probability at least $1-\delta$, the expected total regret of \ef\ can be bounded as:
\begin{eqnarray}
\lefteqn{\bar{R}(\ef,T)} \nonumber \\
&\le& (\rho_1-\rho_0)(E-C^*) + (\rho_3-\rho_0)\delta T \nonumber \\
&\le& (\rho_1-\rho_0)E+(\rho_3-\rho_0)\delta T \nonumber \\
&\le& \frac{\rho_1-\rho_0}{\mu_m}\ln\frac{1}{\mu_m\delta}+(\rho_3-\rho_0)\delta T\,, \label{eqn:ef-exp-regret}
\end{eqnarray}
%
The right-hand side of the last equation is a function of $\delta$, in the form of $f(\delta)\defeq a-b\ln\delta+c\delta$, for $a=\frac{\rho_1-\rho_0}{\mu_m}\ln\frac{1}{\mu_m}$, $b=\frac{\rho_1-\rho_0}{\mu_m}$, and $c=(\rho_3-\rho_0)T$.  Because of convexity of $f$, its minimum is found by solving $f'(\delta)=0$ for $\delta$, giving
\[
\delta^*=\frac{b}{c} = \frac{\rho_1-\rho_0}{(\rho_3-\rho_0)\mu_mT}\,.
\]
Substituting $\delta^*$ for $\delta$ in \eqnref{eqn:ef-exp-regret} gives the desired bound.
\end{proof}

\section{Proofs for \fe}
\label{sec:fe-proofs}

This subsection gives complete proofs for theorems about \fe.  We start with a few technical results that are needed in the main theorem's proofs.

\subsection{Technical Lemmas}
\label{sec:fe-lemmas}

The following general results are the key to obtain our expected regret bounds for \fe.

\begin{lemma} \label{lem:m-loss}
Fix $M\in\Mset$, and let $1 \le t_1 < t_2 < \ldots < t_m \le T$ be the rounds for which $M_t=M$.  Then, the expected total loss incurred in these rounds is bounded as:
\[
\bar{L}_M < (m\rho_0+\rho_2-\rho_3) \bar{L}_1 + (\rho_3-\rho_0)\bar{L}_2 + \rho_1 \bar{L}_3\,,
\]
where
\begin{eqnarray*}
\bar{L}_1 &\defeq& \sum_i\prod_{j<i}(1-\eta_{t_j})\eta_{t_i}\,,\\
\bar{L}_2 &\defeq& \sum_i\prod_{j<i}(1-\eta_{t_j})\eta_{t_i}\cdot i\,,\\
\bar{L}_3 &\defeq& \sum_i\left(\prod_{j<i}(1-\eta_{t_j})\eta_{t_i}\sum_{j>i}\eta_{t_j}\right)\,.
\end{eqnarray*}
\end{lemma}

\begin{proof}
Let $\bar{L}_M(\fe)$ be the expected total loss incurred in the rounds $t$ where $M_t=M$: $1 \le t_1<t_2<\cdots<t_m\le T$ for some $m\ge0$.  Let $I\in\{1,2,\ldots,m,m+1\}$ be the random variable, so that $M$ is first discovered in round $t_I$.  That is,
\[
A_{t_j} = \begin{cases}
	0, & \text{if $j < I$} \\
	1, & \text{if $j = I$}\,.
	\end{cases}
\]
Note that $I=m+1$ means $M$ is never discovered; such a notation is for convenience in the analysis below.
The corresponding loss is given by
\[
(I-1)\rho_3 + \rho_2 + \sum_{j>I}\left(\rho_0\1{A_{t_j}=0} + \rho_1\1{A_{t_j}=1}\right)\,,
\]
whose expectation, conditioned on $I$, is at most
\[
(I-1)\rho_3 + \rho_2 + \sum_{j>I}\left(\rho_0 + \rho_1\eta_{t_j}\right)\,.
\]
Since \fe\ chooses to probe in round $t$ with probability $\eta_t$, we have that
\[
\P{I=i} = \prod_{j<i}(1-\eta_{t_j})\eta_{t_i}\,.
\]
Therefore, $\bar{L}_M(\fe)$ can be bounded by
\begin{align*}
 	& \bar{L}_M \\
\le	& \sum_{i=1}^{m+1}\P{I=i} \left((I-1)\rho_3 + \rho_2 + \sum_{j>I}\left(\rho_0 + \rho_1\eta_{t_j}\right)\right) \\
= & (m\rho_0+\rho_2-\rho_3) \bar{L}_1 + (\rho_3-\rho_0) \bar{L}_2 + \rho_1 \bar{L}_3,,
\end{align*}
where $\bar{L}_1$, $\bar{L}_2$ and $\bar{L}_3$ are given in the lemma statement.
\end{proof}

Now we can obtain the following proposition:

\begin{proposition} \label{prop:fe}
If we run \fe\ with non-increasing exploration rates $\eta_1 \ge \cdots \ge \eta_T>0$, 
then
$$\E[L(\fe,T)] \le \rho_0T + \frac{C^*\rho_3}{\eta_T} + \rho_1 \sum_{t=1}^T \eta_t.
$$
\end{proposition}

\begin{proof}
For each $M\in\Mset$, \lemref{lem:m-loss} gives an upper bound of loss incurred in rounds $t$ for which $M_t=M$:
\[
\bar{L}_M \le (m\rho_0+\rho_2-\rho_3)\bar{L}_1 + (\rho_3-\rho_0) \bar{L}_2 + \rho_1 \bar{L}_3\,,
\]
where $\bar{L}_1$, $\bar{L}_2$ and $\bar{L}_3$ are given in \lemref{lem:m-loss}.  We now bound the three terms of $\bar{L}_M(\fe)$, respectively.

To bound $\bar{L}_1$, we define a random variable $I$, taking values in $\{1,2,\ldots,m,m+1\}$, whose probability mass function is given by
\begin{eqnarray}
\P{I=i} = \begin{cases}
	\prod_{j<i}\left(1-\eta_{t_j}\right)\eta_{t_i}, & \text{if $i \le m$} \\
	\prod_{j\le m}\left(1-\eta_{t_j}\right), & \mbox{if $i = m+1$.}
\end{cases} \label{eqn:aux-rv}
\end{eqnarray}
Therefore, $I$ is like a geometrically distributed random variable, except that the parameter for the $i$th draw is not the same and is $\eta_{t_i}$.  Consequently,
\begin{align*}
\bar{L}_1 = \sum_i \P{I=i} \le 1\,.
\end{align*}

To bound $\bar{L}_2$, we use the same random variable $I$:
\begin{align*}
\bar{L}_2 &= \sum_{i=1}^m \P{I=i} \cdot i \\
  &\le \sum_{i=1}^m \P{I \ge i} \mathcomment{Corollary of Theorem~3.2.1 of \cite{Chung00Course}} \\
  &= \sum_{i=1}^m \prod_{j<i} (1-\eta_{t_j}) \mathcomment{By definition of $I$ in \eqnref{eqn:aux-rv}} \\
  &\le \sum_{i=1}^m \prod_{j<i} (1-\eta_{t_T}) \mathcomment{By assumption that $\eta_1\ge\cdots\ge\eta_T$} \\
  &= \frac{1}{\eta_T}\left(1-(1-\eta_T)^m\right) \le \frac{1}{\eta_T}\,.
\end{align*}

To bound $\bar{L}_3$, we have
\begin{align*}
\bar{L}_3 &\le \sum_{i=1}^m\prod_{j<i}(1-\eta_{t_j})\eta_{t_i}\sum_{j=1}^m \eta_{t_j}
= \bar{L}_1 \sum_{j=1}^m\eta_{t_j} \le \sum_{j=1}^m\eta_{t_j}\,.
\end{align*}

Putting all three bounds above, we have
\[
\bar{L}_M(\fe) \le m \rho_0 + \frac{\rho_3-\rho_0}{\eta_T} + \rho_1 \sum_{j=1}^m\eta_{t_j}\,.
\]

Now sum up all $\bar{L}_M(\fe)$ over all $M\in\Mset$ that appear in the sequence $(M_t)_t$, and we have
\[
\E[L(\fe,T)] \le \rho_0 T + \frac{C^* \rho_3}{\eta_T} + \rho_1 \sum_{i=1}^T \eta_{t_i}\,.
\]
\end{proof}

\subsection{Proof for \thmref{thm:fe-main}}
\label{sec:fe-main}

\oldthm{thm:fe-main}
\textit{\textThmFeMain}
\thmskip

\begin{proof}
The proof is split into two parts, for the two stated bounds.

\noindent\textbf{High-probability Regret Bound.}
Fix any $M\in\Mset$, and let $1 \le t_1 < t_2 < \ldots < t_m \le T$ be the rounds for which $M_t=M$.  Then, for any $m' \le m$, we can upper-bound the probability that $M$ remains undiscovered after the first $m'$ rounds for which $M_t=M$:
\[
\P{M \notin \Mset_{t_{m'}+1}} = \prod_{i=1}^{m'} (1-\eta_{t_i}) < \exp(-\sum_{i=1}^{m'}\eta_{t_i})\,,
\]
where the inequality is due to the fact that $1-x \le e^{-x}$.  We will show that for sufficiently large $m'$, the right-hand side above, $\exp(-\sum_{i=1}^{m'}\eta_{t_i})$, is at most $\delta/C^*$; in other words, with probability at least $1-\delta/C^*$, item $M$ will be discovered after appearing $m'$ times for sufficiently large $m'$.  Indeed,
\begin{eqnarray*}
\sum_{i=1}^{m'}\eta_{t_i}
  &\ge& \sum_{t=T-m'+1}^T \eta_t \mathcomment{monotonicity of $(\eta_t)_t$} \\
  &=& \sum_{t=T-m'+1}^T t^{-\alpha} \mathcomment{definition of $\eta_t$} \\
  &\ge& \int_{T-m'+1}^T t^{-\alpha}dt \\
  &=& \frac{1}{1-\alpha} \left(T^{1-\alpha} - (T-m'+1)^{1-\alpha}\right) \\
  &\ge& \frac{1}{1-\alpha} \left.\frac{d}{dt}t^{1-\alpha}\right|_{t=T} \cdot (T-(T-m'+1)) \\
  & & \text{(concavity of $t^{1-\alpha}$)} \\
  &=& T^{-\alpha}(m'-1)\,.
\end{eqnarray*}
Therefore, we will have $\P{M \notin \Mset_{t_{m'}+1}} \le \delta/C^*$ if $T^{-\alpha}(m'-1) \ge \ln\frac{C^*}{\delta}$, or equivalently, $m' \ge T_0$, where
\[
T_0 = T^\alpha\ln\frac{C^*}{\delta}+1\,.
\]
It follows that, with probability at least $1-\delta/C$, the total loss associated with item $M$ (that is, the total loss accumulated in $\{t_1,t_2,\ldots,t_m\}$) is at most
\begin{equation}
\rho_3 \min\{m,T_0\} + \rho_0(m-T_0)_+\,, \label{eqn:fe-hp-m-loss}
\end{equation}
where $(x)_+\defeq\max\{x,0\}$.

Define $\Mset^*\defeq\{M_1,\ldots,M_T\}\subseteq\Mset$ be the set of types that appear in the sequence $(M_t)_t$.  Clearly, $C^*=\setcard{\Mset^*}$.  Summing \eqnref{eqn:fe-hp-m-loss} over all $M\in\Mset^*$ and applying a union bound, we have the following that holds with probability at least $1-\delta$:
\begin{eqnarray}
\lefteqn{L(\fe,T)} \nonumber \\
  &\le& \sum_{M\in\Mset^*} \big( \rho_3 \min\{m(M),T_0\} + \rho_0(m(M)-T_0)_+ \big) \nonumber \\
  &\le& C^* \rho_3 T_0 + \sum_{M\in\Mset^*} \rho_0(m(M)-T_0)_+\,, \label{eqn:fe-hp-loss}
\end{eqnarray}
where $m(M) \defeq \setcard{\{1 \le t \le T \mid M_t=M\}}$ is the number of times $M$ appears in $T$ rounds.  Now consider the optimal yet hypothetical strategy, whose total loss, given in \eqnref{eqn:optimal-loss}, can be written as
\begin{eqnarray}
L^*(T)\,\, = & \sum_{M\in\Mset^*}& \big(\rho_2 + \rho_0 (\min\{m(M),T_0\}-1) \nonumber \\
&& + \rho_0(m(M)-T_0)_+ \big)\,. \label{eqn:optimal-loss-again}
\end{eqnarray}
In \eqnref{eqn:optimal-loss-again}, for each $M\in\Mset^*$, the first two terms correspond to the loss accumulated in the first $\min\{m(M),T_0\}$ times where $M_t=M$, and the last term for the remaining rounds where $M_t=M$.  It then follows from Equations~\ref{eqn:fe-hp-loss} and \ref{eqn:optimal-loss-again} that, with probability at least $1-\delta$,
\[
R(\fe,T) \le C^*\rho_3 T_0 = C^*\rho_3(T^\alpha\ln\frac{C^*}{\delta}+1)\,,
\]
as stated in the theorem.

\noindent\textbf{Expected Regret Bound.}
Given polynomial exploration rates $\eta_t=t^{-\alpha}$, we have
\begin{eqnarray*}
\sum_{t=1}^T \eta_t &=& 1 + \sum_{t=2}^T t^{-\alpha} \\
&\le& 1 + \int_1^T t^{-\alpha} dt \\
&=& 1 + \frac{1}{1-\alpha}\left. t^{1-\alpha}\right|_{t=1}^T \\
&=& 1 + \frac{1}{1-\alpha}\left(T^{1-\alpha}-1\right) \\
&\le& \frac{T^{1-\alpha}}{1-\alpha}\,.
\end{eqnarray*}
The total regret follows immediately from \propref{prop:fe}.  Furthermore, if one sets $\alpha=1/2$, the regret bound becomes $(C^*\rho_3+2\rho_1)\sqrt{T}=O(\sqrt{T})$.
\end{proof}

\section{Proof for \thmref{thm:lb}}
\label{sec:proof-lb}

\oldthm{thm:lb}
\textit{\textThmLb}

\begin{proof}
We construct a stochastic OCCP with $\Mset=\{1,2,\ldots,C\}$ and distribution $\mu$ so that
\[
\mu(M) = \begin{cases} \mu_m, & \mbox{if $M < C$} \\ 1-C\mu_m, & \mbox{if $M=C$}\,, \end{cases}
\]
where $\mu_m=1/\sqrt{T} \ll 1$.  For every $M\in\Mset$, define $T_M\in\{1,\ldots,T,\infty\}$ as the first time $M$ is collected; that is
\[
T_M \defeq \min \{t \mid M_t=M, A_t=\actProbe\}\,,
\]
with the convention that $T_M=\infty$ if $\{t \mid M_t=M, A_t=\actProbe\}=\emptyset$.  Furthermore, let $1\le t_1 < t_2 < \cdots < t_E \le T$ be the rounds in which probing ($A_t=\actProbe$) happens; denote by $\Eset$ the set $\{t_1,t_2,\ldots,t_E\}$.  Since the two random variables $M_t$ and $A_t$ are independent, we have for any $i \in \{1,2,\ldots,E\}$ and any $M\in\Mset$ that
\[
\P{M_{t_i}=M} = \mu(M)\,.
\]

We start with the expected-regret lower bound and let $\Aalg$ be any admissible algorithm.  Conditioning on $\Eset$ being the rounds of probing, we want to lower bound the number $E$ of exploration rounds so that the probability of \emph{not} discovering \emph{all} items in $\Mset$ is at most $\delta$ (which is necessary for the expected regret to be $O(\sqrt{T})$).  First, note that the events $\{T_M<\infty\}_{M\in\Mset}$ are negatively correlated, since discovering some $M_1$ in $\Eset$ can only decrease the probability of discovering $M_2\ne M_1$ in $\Eset$.  Therefore, we have
\begin{eqnarray*}
\P{\forall M, T_M < \infty} &\le& \prod_{M\in\Mset}\P{T_M<\infty} \\
&\le& \prod_{M=1}^{C-1} \P{T_M<\infty} \\
&\le& (1-(1-\mu_m)^E)^{C-1}\,.
\end{eqnarray*}
Making the last expression to be $1-\delta$, we have
\begin{align*}
E &= \frac{\ln\left(1-(1-\delta)^{\frac{1}{C-1}}\right)}{\ln(1-\mu_m)} \\
  &= \Omega\left(\frac{\ln\left(1-(1-\frac{\delta}{C})\right)}{-\mu_m}\right) \\
  &= \Omega\left(\frac{1}{\mu_m}\ln\frac{C}{\delta}\right)\,,
\end{align*}
for sufficiently small $\mu_m$ and $\delta$.

For simplicity, assume $\rho_0=0$ without loss of generality; otherwise, we can just define a related problem with $\rho_i'\defeq\rho_i-\rho_0$, where the loss is just shifted by a constant and the regret remains unchanged.  With this assumption, the optimal expected loss given in \eqnref{eqn:optimal-loss} becomes $L^*(T)=C^*\rho_2$.

With $\rho_0=0$, the expected loss of $\Aalg$ is at least $(E-C^*)\rho_1+C^*\rho_2+\delta(T-E)\mu_m\rho_3$, where the first two terms are for the loss incurred during the $E$ probing rounds; and the last term for the $\delta$-probability event that some item is not discovered in the probing rounds, which leads to $\rho_3$ loss when it is encountered in any of the remaining $T-E$ rounds.

The regret of $\Aalg$, by comparing its loss to $L^*(T)$, can be lower bounded by
\[
(E-C^*)\rho_1+\delta(T-E)\mu_m\rho_3 = \Omega\left(\rho_3T\mu_m+\frac{\rho_1}{\mu_m}\ln\frac{C^*}{\delta}\right)\,,
\]
giving an expected-regret lower bound by observing the fact that $\mu_m=1/\sqrt{T}$.

The high-probability lower bound can be proved by very similar calculations, with the observation that all $C$ types need to be collected in order to have a regret bound that holds with probability $1-\delta$, for sufficiently small $\delta$.
\end{proof}

\section{Algorithm Pseudocode}
\label{app:alg}

The following algorithm, \pacexp\ of \namecite{Guo15Concurrent}, is a key component in \algref{alg:eee-fe}.  It takes as input two parameters: threshold $m$ for determining a state--action pair is known or not, and planning horizon $L$ that is used to compute an exploration policy.

\begin{algorithm}[h]
\caption{\pacexp\ of \namecite{Guo15Concurrent}} \label{alg:pacexp}
\begin{algorithmic}[1]
\item \textbf{Input:} $m$ (known threshold), $L$ (planning horizon)
\WHILE{some $(s,a)$ has not been visited at least $m$ times}
\STATE Let $s$ be the current state
\IF{all $a$ have been tried $m$ times}
\STATE{Start a new $L$-step episode}
\STATE{Construct an empirical known-state MDP $\hat{M}_K$ with the reward of all known $(s,a)$ pairs set to $0$, all unknown set to $1$ (maximum reward value), the transition model of all known $(s,a)$ pairs set to the estimated parameters and the unknown to self loops}
\STATE{Compute an optimistic $L$-step policy $\hat{\pi}$ for $\hat{M}_K$}
\STATE{From the current state, follow $\hat{\pi}$ for $L$ steps, or until an unknown state is reached}
\ELSE
\STATE{Execute $a$ that has been tried the least}
\ENDIF
\ENDWHILE
\end{algorithmic}
\end{algorithm}

\section{Proofs for LLRL Sample Complexity}
\label{sec:llrl-proofs}

This section provides details of the sample-complexity analysis of \algref{alg:eee-fe}, leading to the main result of \thmref{thm:llrl-fe}.

\subsection{Proof of \lemref{lem:diameter}}
\label{app:diameter-proof}

\oldlem{lem:diameter}
\textLemDiameter

\begin{proof}
The proof follows closely to that of \namecite{Guo15Concurrent}.  Consider the beginning of an episode, and let $K$ be the set of known state--action pairs which have been visited by the agent at least $m$ times.  For each $(s,a)\in k$, the $\ell_1$ distance between the empirical estimate and the actual next-state distribution is at most (Lemma~8.5.5 of \namecite{Kakade03Sample}): $\alpha=\sqrt{\frac{8N}{m}\log\frac{2N}{\delta}}$. 
Let $M_K$ be the known-state MDP, which is identical to $\hat{M}_K$ except that the transition probabilities are replaced by the true ones for known state--action pairs.  Following the same line of reasoning as \namecite{Guo15Concurrent}, one may lower-bound the probability that an unknown state is reached within the episode by
$p_e \ge 1/6 - 3\alpha D$.
Therefore, $p_e$ is bounded by $1/12$ as long as $\alpha D \le 1/36$.  The latter is guaranteed if
$m \ge m_0 = O\left( ND^2\log\frac{N}{\delta} \right)$.
The rest of the proof is the same as \namecite{Guo15Concurrent}, invoking Lemma~56 of \namecite{Li09Unifying} to get an upper bound of $H$, stated in the lemma as $H_0(m)$.
\end{proof}

\subsection{Proof of \lemref{lem:correct-discovery}}
\label{sec:correct-discovery-proof}

\oldlem{lem:correct-discovery}
\textLemCorrectDiscovery

\begin{proof}
For task $M_t$, let $\Event_t$ be the event that all state--action pairs become known after $H$ steps; \lemref{lem:diameter} with a union bound shows all events $\{\Event_t\}_{t\in\{1,2,\ldots,T\}}$ hold with probability at least $1-\delta$.  For every fixed $t$, under event $\Event_t$, every state--action pair has at least $m$ samples to estimate its transition probabilities and average reward after $H$ steps.  Applying Lemma~8.5.5 of \namecite{Kakade03Sample} on the transition distribution, we can upper bound, with probability at least $1-\frac{\delta}{2SAT}$, the $\ell_1$ error of the transition probability estimates by:
\[
\epsilon_T = \sqrt{\frac{8N}{m}\log\frac{4SAT}{\delta}}\le\frac{\Gamma}{3}\,.
\]
Similarly, an application of Hoeffding's inequality gives the following upper bound, with probability at least $1-\frac{\delta}{2SAT}$, on the reward estimate:
\[
\epsilon_R = \sqrt{\frac{2}{m}\log\frac{4SAT}{\delta}}\le\frac{\Gamma}{6\sqrt{N}}\,.
\]
Applying a union bound over all states, actions, and tasks, the above concentration results hold with probability at least $1-\delta$ for an agent running on $T$ tasks.  The rest of the proof is to show that task identification succeeds when the above concentration inequalities hold.

To do this, consider the following two mutually exclusive cases:
\begin{compactenum}
\item{If $M_t$ is new, then, by assumption, for every $M'\in\hat{\Mset}$, there exists some $(s,a)$ for which the two models differ by at least $\Gamma$ in $\ell_2$ distance; that is, $\|\theta_{M_t}(\cdot|s,a)-\theta_{M'}(\cdot|s,a)\|_2 \ge \Gamma$.  It follows from the equality,
\begin{eqnarray*}
\lefteqn{\|\theta_{M_t}(\cdot|s,a)-\theta_{M'}(\cdot|s,a)\|_2^2} \\
&=& \sum_{1 \le s' \le S} \left(\theta_{M_t}(s'|s,a)-\theta_{M'}(s'|s,a)\right)^2 \\
& & \mathcomment{error in transition probability estimates} \\
& & + \left(\theta_{M_t}(S+1|s,a)-\theta_{M'}(S+1|s,a)\right)^2\,, \\
& & \mathcomment{error in reward estimate}
\end{eqnarray*}
that at least one of two terms on the right-hand side above is at least $\Gamma^2/2$.

If the first term is larger than $\Gamma^2/2$, then the $\ell_1$ distance between the two next-state transition distributions is at least $\Gamma/\sqrt{2}$, which is larger than $2\epsilon_T=2\Gamma/3$.  It implies that the $\ell_1$-balls of transition probability estimates for $(s,a)$ between $M_t$ and $M'$ do not overlap, and we will identify $M_t$ as a new MDP.
Similarly, if the second term is larger than $\Gamma^2/2$, then using $\epsilon_R$ we can still identify $M_t$ as a new MDP.}

\item{If $M_t$ is not new, we claim that the algorithm will correctly identify it as some previously solved MDP, say $M''\in\hat{\Mset}$.  In particular, confidence intervals of its estimated model in every state--action pair must overlap with $M''$, since both models' confidence intervals contain the true model parameters.  On the other hand, for any $M'\in\hat{\Mset}\setminus\{M''\}$, its model estimate's confidence intervals do not have overlap with that of $M_t$'s in at least one state--action pair, as shown in case~1.  Therefore, the algorithm can find the unique and correct $M''\in\hat{\Mset}$ that is the same as $M_t$.}
\end{compactenum}

Finally, the lemma is proved with a union bound over all tasks, states and actions, and with the probability that $E_t$ fails to hold for some $t$.
\end{proof}

\subsection{Proof of \thmref{thm:llrl-fe}}
\label{sec:llrl-fe-proof}

\oldthm{thm:llrl-fe}
\textThmLlrlfe

\begin{proof}
We consider each possible case when solving the $t$th task, $M_t$.  As shown in \lemref{lem:correct-discovery}, with probability $1-\delta$, the following event $\Event_t$ hold for all $t\in[T]$: after \pacexp\ is run on $M_t$, \algref{alg:eee-fe} will discover the identity of $M_t$ correctly.  That is, if $M_t$ is a new MDP, it will be added to $\hat{\Mset}$; otherwise, $\hat{\Mset}$ remains unchanged.  In the following, we assume $\Event_t$ holds for every $t$, and consider the following cases:

\paragraph{(a) Exploitation in discovered tasks:}
we choose to exploit (line~\ref{alg:eee-fe:exploit} in Alg~\ref{alg:eee-fe}) 
and $M_t$ has been already discovered.  In this case, \fmllrl\ is used to do model elimination (within $\hat{\Mset}$) and to transfer samples from previous tasks that correspond to the same MDP as the current task $M_t$.  Therefore, with a similar analysis, we can get a per-task sample complexity of at most $O(CDm)= \tilde{O}(\frac{CD}{\Gamma^2})=\rho_0$.

\paragraph{(b) Exploitation in undiscovered tasks:}
we choose to exploit 
and $M_t$ has \emph{not} been discovered.  Running \fmllrl\ in this case can end up with an arbitrarily poor policy which follows a non-$\epsilon$-optimal policy in every step.  Therefore, the sample complexity can be as large as $H=\rho_3$.

\paragraph{(c) Exploration:} we choose to explore using \pacexp\ (lines~\ref{alg:eee-fe:explore:start}--\ref{alg:eee-fe:explore:end} in Alg~\ref{alg:eee-fe}).  In this case, with high probability, it takes at most $H_0(m)$ steps to make every state known, so that the model parameters can be estimated to within accuracy $O(\Gamma)$.
After that, we can reliably decide whether $M_t$ is a new MDP or not.  With sample transfer, the additional steps where $\epsilon$-sub-optimal policies are taken in the MDP corresponding to $M_t$ (accumulated across all tasks in the $T$-sequence) is at most $\zeta_s$, the single-task sample complexity.  The total sample complexity for tasks corresponding to this MDP is therefore at most $H_0(m) T(M_t) + \zeta_s=\rho_2 T(M_t)+\zeta_s$, where $T(M_t)$ is the number of times this MDP occurs in the $T$-sequence.

Finally, when \algref{alg:eee-fe} is run on a sequence of $T$ tasks, the total sample complexity---the number of steps in all tasks for which the agent does not follow an $\epsilon$-optimal policy---is given by one of the three cases above.  The sample complexity of exploration can therefore be upper bounded by adding \eqnref{eqn:optimal-loss} to \eqnref{eqn:fe-hp-regret} in \thmref{thm:fe-main}, completing the proof with an application of union bound that takes care of error probabilities (those involved in \lemref{lem:correct-discovery}, in upper-bounding sample complexity in individual tasks in the proof above, and in \thmref{thm:fe-main}).
\end{proof}

\section{Experiment Details}
\label{sec:exp_det}

\subsection{Gridworld}

For the grid world domain, 
all four MDPs had the same $25$-cell square grid layout and 
$4$ actions (up, down, left, right). State s1 is in the upper 
left hand corner, state s5 is the upper right hand corner, 
s20 is the lower left hand corner, and s25 is the lower right 
hand corner. All other states are labeled sequentially between 
these.
Actions succeed in their intended direction with 
probability $0.85$ and with probability $0.05$ go in each the other three directions (unless halted by a wall when the agent stays in the same state). For all actions corner states $s5$, $s20$, and $s25$ 
stay in the same state with probability $0.95$ or transition back to the 
start state (for all actions). The start state is at the center 
of the square grid ($s13$). The dynamics of all MDPs are identical. 
All rewards are sampled from Bernoulli distributions.  All rewards 
have parameter $0.0$ unless otherwise noted: \\
In MDP 1, corner state 
$s20$ has a reward parameter of $0.75$. In MDP 2, corner state 
$s5$ has a reward parameter of $0.75$.  In MDP 3, corner state 
$s25$ has a reward parameter of $0.75$.  In MDP 4, corner state 
$s25$ has a reward parameter of $0.75$, and corner state $s1$ has a reward parameter of $0.99$. 

\ef\ is given an upper bound on the number of MDPs ($4$) and 
the minimum probability of any of the MDPs across the $100$ tasks. 
When we compared to the Bayesian hierarchical multi-task learning 
algorithm \hmtl\ for the stochastic setting, we also provided it 
with an upper bound on the number of MDPs, though \hmtl\ is 
also capable of learning this directly. We used \hmtl\ with a 
two-level hierarchy (e.g. a class consists of a single MDP).  
We ran a variant of \fe\ labeled ``ForcedExp'' in the 
figures which 
uses a polynomially decaying exploration rate, $t^{\alpha}$ with $\alpha=0.5$, 
for all experiments. Performance does vary with the choice of 
$\alpha$ but 
$\alpha=0.5$ gave good results in our 
preliminary investigations.  Interestingly, this is consistent with
 the theoretical result that $\alpha=0.5$ minimizes dependence on $T$ for polynomially 
decaying exploration rates (\cf, \thmref{thm:fe-main}). 

We also explored the \fe\ algorithm 
using a constant exploration rate $\frac{2}{\sqrt{T}}$ for some earlier 
experiments: as expected performance was similar but slightly worse 
generally than using a decaying exploration rate, and so we focus 
all comparisons on the decaying exploration rate variant.

\subsection{Simulated Human-Robot Collaboration}

Our abstracted human-robot collaboration simulation comes 
from the recent work of \namecite{nikolaidis2015efficient}. 
The authors showed significant benefits in a human-robot collaboration problem, by assuming that user preference models over human-robot collaboration could be 
clustered into a small set of types. In their work, they took a previously 
collected set of data, and clustered it using the Expectation-Maximization (EM) algorithm
into a set of $4$ user types. Then, for each new user, they 
treated the problem as a mixed observability Markov 
decision process, where the (static) hidden state is 
the type of the user. In contrast to their work, we 
handle online lifelong learning across tasks, 
and our central contribution is a formal analysis of 
the sample complexity and performance, as opposed 
to \namecite{nikolaidis2015efficient} that present exciting empirical results on real human-robot interactions, without a theoretical analysis.

To demonstrate that our approach could also achieve 
good performance for this setting, we performed simulation 
experiments by constructing a lifelong learning 
domain in which each task is sampled from the four MDP 
models learned by \namecite{nikolaidis2015efficient}.\footnote{We thank 
those authors for sharing their models.}

The domain involves a human and robot collaborating to 
paint a box. The box is defined by its location along 
the horizontal ($5$ positions) and vertical ($11$ positions) axes, as well as its 
tilt angle ($11$ values), for a total of $605$ states. 
The possible actions of the robot are to change each of the three dimensions 
of the box's location (to stay the same or move forward or 
backward along that axis), resulting in $3^3=27$ actions. 
The transition dynamics are deterministic and identical 
for all $4$ MDP models. The MDP models differ in their (deterministic) reward 
models.  \namecite{nikolaidis2015efficient}  
learned the MDP models using the EM algorithm and 
inverse reinforcement learning from a set of $15$ humans 
performing $4$ different variants of the human-robot box 
painting task (varying by which position the human performed 
the task in) where the robot annotates actions for the robot.\footnote{Inverse reinforcement learning is then used to infer a reward function of the human user that would make the actions prescribed by the human for the robot optimal.}
We introduced a small amount of Gaussian 
noise (with $0.01$ standard deviation and zero mean) 
to the rewards. Note that even if the models are known 
to be deterministic, an agent learning with no prior 
information must still visit all $S\times A = 16335$ 
state--action pairs at least once to learn their dynamics, 
and of course it is not always possible to directly reach 
any other state in a single action. 

In our simulation, for each task one of the $4$ MDPs was randomly 
selected, and the agent executed in it for $H$ steps without a priori 
knowledge of its identity.  We set the horizon length by $H=3SA=49005$, 
so that it was feasible to visit all state--action pairs at least once.

We tested our \ef\ algorithm on this domain with a total number 
of tasks per ``run'' as $80$. We report results averaged over 
$30$ runs.

\end{document}